\documentclass{article}

\usepackage[utf8]{inputenc}
\usepackage{graphicx}
\usepackage{subfigure}
\usepackage{amsmath}
\usepackage{amssymb}
\usepackage{amsthm}
\usepackage{geometry}
\usepackage{comment}
\usepackage{url}
\newcommand{\R}{\mathbb{R}}

\newcommand{\N}{\mathbb{N}}

\newcommand{\Dg}{\mathrm{Dg}}

\newcommand{\limitinf}[1]{\underset{{#1}\rightarrow+\infty}{\rm lim}}
\newcommand{\card}{{\rm card}}

\newcommand{\HS}{\mathcal{H}}

\DeclareMathOperator{\dimA}{dim_A}

\newcommand{\SpD}{{\mathcal D}}

\newcommand{\SpLD}{{\mathcal D}^L}

\newcommand{\SpND}{{\mathcal D}_N}

\newcommand{\SpNLD}{{\mathcal D}_N^L}

\newtheorem{thm}{Theorem}[section]

\newtheorem{prop}[thm]{Proposition}
\newtheorem{lem}[thm]{Lemma}
\newtheorem{defin}[thm]{Definition}
\newtheorem{rem}[thm]{Remark}

\title{On the Metric Distortion of Embedding Persistence Diagrams into separable Hilbert spaces}
\author{Mathieu Carri\`ere\footnote{mc4660@columbia.edu} and Ulrich Bauer\footnote{ulrich.bauer@tum.de}}
\date{}
\begin{document}

\maketitle

\begin{abstract}
Persistence diagrams are important descriptors in Topological Data Analysis.
Due to the nonlinearity of the space of persistence diagrams equipped with their {\em diagram distances}, 
most of the recent attempts at using persistence diagrams in machine learning have been done through 
kernel methods, i.e., embeddings of persistence diagrams into Reproducing Kernel Hilbert Spaces, 
in which all computations can be performed easily. Since persistence diagrams enjoy theoretical stability 
guarantees for the diagram distances, the {\em metric properties} of the feature map, i.e., the relationship 
between the Hilbert distance and the diagram distances, are of central interest for understanding if the 
persistence diagram guarantees carry over to the embedding. In this article, we study the possibility of embedding 
persistence diagrams into separable Hilbert spaces, with bi-Lipschitz maps. In particular, 
we show that for several stable embeddings into infinite-dimensional Hilbert spaces defined in the literature, 
any lower bound must depend on the cardinalities of the persistence diagrams,
and that when the Hilbert space is finite dimensional, finding a bi-Lipschitz embedding is impossible, even when restricting 
the persistence diagrams to have bounded cardinalities.
\end{abstract}

\section{Introduction}
\label{sec:intro}

The increase of available data in both academia and industry have been exponential over the past few decades,
making data analysis ubiquitous in many different fields of science. Machine learning has proved to be
one of the most prominent field of data science, leading to astounding results in various applications, such as
image and signal processing. Topological Data Analysis (TDA)~\cite{Carlsson09a} is one specific field of machine learning, which
focuses more on {\em complex} rather than big data. The general assumption of TDA is that data is actually sampled 
from geometric or low-dimensional domains, whose geometric features are relevant to the analysis.
These geometric features are usually encoded in a mathematical object called {\em persistence diagram},
which is roughly a set of points in the plane, each point representing a topological feature whose size is 
contained in the coordinates of the point. Persistence diagrams have been proved to bring complementary
information to other traditional descriptors in many different applications, often leading to large result improvements. 
This is also due to the so-called {\em stability properties} of the persistence diagrams, which state that persistence diagrams computed on similar data
are also very close in the diagram distances~\cite{Cohen07,Bauer13b,Chazal16a}.   

Unfortunately, the use of persistence diagrams in machine learning methods is not straightforward, since many algorithms expect data to be 
Euclidean vectors, while persistence diagrams are sets of points with possibly different cardinalities. Moreover, the {\em diagram distances}
used to compare persistence diagrams are computed with optimal matchings, and thus quite different from Euclidean metrics.
The usual way to cope with such difficult data is to use {\em kernel methods}. A kernel is a symmetric function on the data 
whose evaluation on a pair of data points equals the scalar product of the images of these points 
under a \emph{feature map} into a Hilbert space, called the {\em Reproducing Kernel Hilbert Space} of the kernel. 
%Both the feature map and the RKHS are usually not given explicitly. 
Many algorithms can be {\em kernelized}, such as PCA and SVM, allowing one to handle non-Euclidean data as soon as a kernel or a feature map is available. 
%Moreover, a seminal result in kernel methods states that the only requirement for a symmetric function to be a kernel is to be positive definite.
%Another possibility is to use Berg's theorem~\cite{Berg84}, which states that a Gaussian computed with a given symmetric function
%is a kernel if the function is {\em conditionnally negative definite}. 

Hence, the question of defining a feature map into a Hilbert space has been
intensively studied in the past few years, and, as of today, various methods can be implemented,
either into finite or infinite dimensional Hilbert spaces~\cite{Bubenik15, Carriere15a, Reininghaus15, Kusano16, Adams17, Carriere17e, Hofer17}.  
Since persistence diagrams are known to enjoy stability properties, it is also natural to ask the same guarantee for their 
embeddings. Hence, all feature maps defined in the literature satisfy a stability property stating that the Hilbert
distance between the image of the persistence diagrams is upper bounded by the diagram distances. A more difficult question is to
prove whether a lower bound also holds or not. Even though one attempt has already been made to show such a lower bound
for the so-called Sliced Wasserstein distance in~\cite{Carriere17e}, the question remains open in general.

\paragraph*{Contributions.} In this article, we tackle the general question of defining
bi-Lipschitz embeddings of persistence diagrams into separable Hilbert spaces. More precisely, we show that:

\begin{itemize}
\item For several stable feature maps defined in the literature, if such a bi-Lipschitz embedding exists, then 
%for a separable (Theorem~\ref{th:sepa}), or Gaussian-like (i.e., obtained with Berg's theorem) RKHS (Theorem~\ref{th:NonEqCnsd}), 
the lower bound goes to 0 or the upper bound goes to $+\infty$ as the number of points and their coordinates increase in the persistence diagrams 
(Theorem~\ref{thm:th1} and Proposition~\ref{prop:fmap}). 
%thus generalizing the result of~\cite{Carriere17e}.
\item Such a bi-Lipschitz embedding does not exist if the Hilbert space is finite dimensional (Theorem~\ref{th:nonembedRn}),
\end{itemize} 

Finally, we also provide experimental evidence of this behavior by computing the metric distortions of various feature maps for persistence diagrams with increasing cardinalities.

\paragraph*{Related work.}
Feature maps for persistence diagrams can be classified into two different classes, depending whether the corresponding Hilbert space is 
finite or infinite dimensional.

In the infinite dimensional case, the first attempt was that proposed in~\cite{Bubenik15}, in which persistence diagrams are turned into
$L^2$ functions, called Landscapes, by computing the homological rank functions given by the persistence diagram points.
Another common way to define a feature map is to see the points of the persistence diagrams as centers
of Gaussians with a fixed bandwidth, weighted by the distance of the point to the diagonal. This is the approach
originally advocated in~\cite{Reininghaus15}, and later generalized in~\cite{Kusano17}, leading to the so-called
{\em Persistence Scale Space} and {\em Persistence Weighted Gaussian} feature maps.
Another possibility is to define a Gaussian-like feature map by using the {\em Sliced Wasserstein distance} between persistence diagrams, which
is conditionnally negative definite. This implicit feature map, called the {\em Sliced Wasserstein} map, was defined in~\cite{Carriere17e}.

In the finite dimensional case, many different possibilities are available. One may consider evaluating a family of tropical polynomials 
onto the persistence diagram~\cite{Verovsek16}, taking the sorted vector of the pairwise distances
between the persistence diagram points~\cite{Carriere15a}, or computing the coefficients of a complex polynomial whose roots are given by the persistence diagram points~\cite{diFabio15}.
Another line of work was proposed in~\cite{Adams17} by discretizing the Persistence Scale Space feature map. The idea is to discretize the plane into a fixed grid, and then compute
a value for each pixel by integrating Gaussian functions centered on the persistence diagram points. Finally, persistence diagrams have been incorporated in deep learning
frameworks in~\cite{Hofer17}, in which Gaussian functions (whose means and variances are optimized by the neural network during training) 
are integrated against persistence diagrams seen as discrete measures.

\section{Background}
\label{sec:background}

\subsection{Persistence Diagrams}

{\em Persistent homology} is a technique of TDA coming from topological algebra that allows the user to compute and encode topological information 
of datasets in a compact descriptor called the {\em persistence diagram}. Given a dataset $X$, often given in the form of a point cloud in $\R^n$,
and a continuous and real-valued function $f:X\rightarrow\R$, the persistence diagram of $f$ can be computed under mild conditions (the function has to be {\em tame}, 
see~\cite{Chazal16a} for more details), and consists in a finite set of points with multiplicities 
in the upper-diagonal half-plane $\Dg(f)=\{(x_i,y_i)\}\subset \{(x,y)\in\R^2:y > x\}$. 
This set of points is computed from the family of 
{\em sublevel sets} of $f$, that is the sets of the form $f^{-1}((-\infty,\alpha])$, for some $\alpha\in\R$. 
More precisely, persistence diagrams encode the different {\em topological events} that occur as $\alpha$ increases from $-\infty$ to $+\infty$. 
Such topological events include creation and merging of connected components and cycles in every dimension; see Figure~\ref{fig:examplePersistence}.
Intuitively, persistent homology records, for each topological feature that appears in the family of sublevel sets, the value $\alpha_b$
at which the feature appears, called the {\em birth value}, and the value $\alpha_d$ at which it gets merged or filled in, called the {\em death value}.
These values are then used as coordinates for a corresponding point in the persistence diagram. Note that 
several features may have the same birth and death values, so points in the persistence diagram have multiplicities. 
Moreover, since $\alpha_d\geq \alpha_b$, these points are always
located above the diagonal $\Delta=\{(x,x):x\in\R\}$. A general intuition about persistence diagrams is that the distance of a point to $\Delta$ is a direct measure of its relevance: if a point is close to $\Delta$, it means that the corresponding cycle got filled in right after its appearance, thus suggesting that
it is likely due to noise in the dataset. On the contrary, points that are far away from $\Delta$ represent cycles with a significant life span, and are more likely to be
relevant for the analysis.  
We refer the interested reader to~\cite{Edelsbrunner10,Oudot15} for more details about persistent homology. 

\begin{figure}
\centering
\includegraphics[width=14cm]{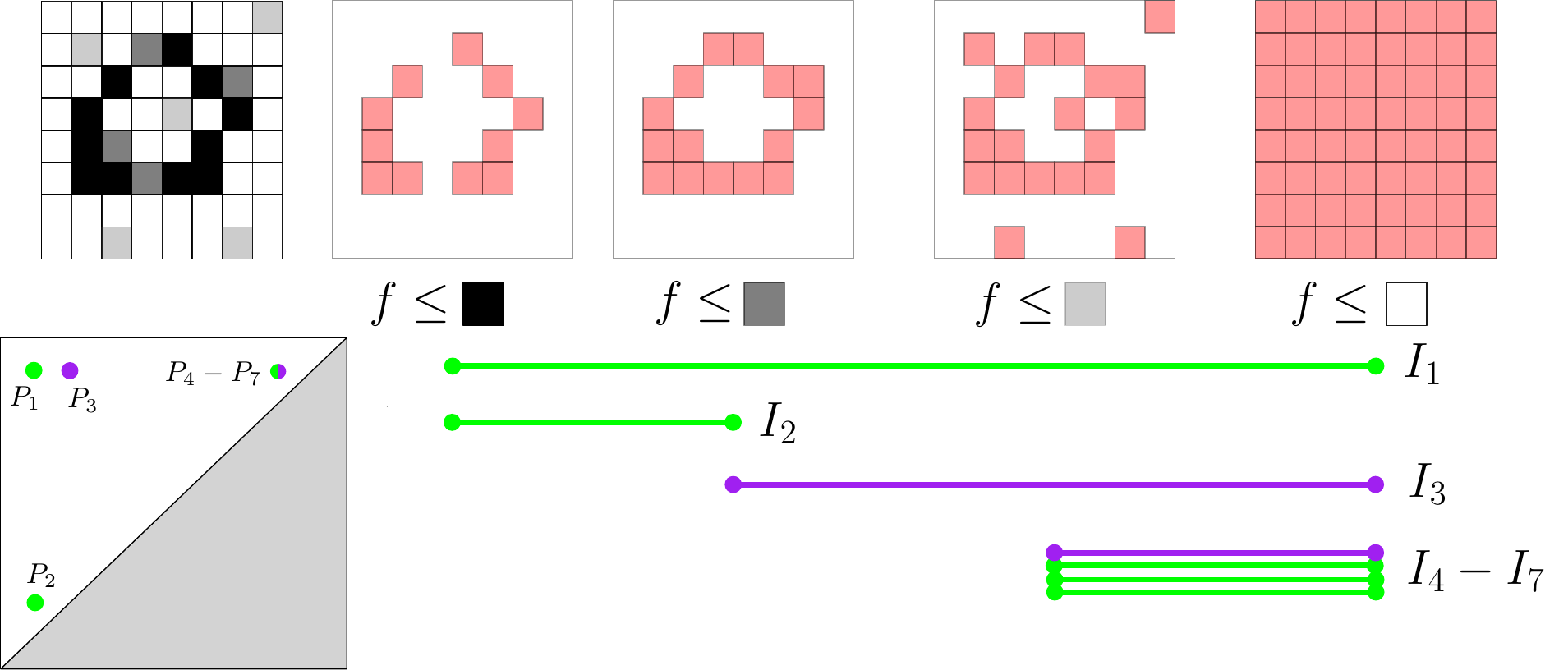}
\caption{\label{fig:examplePersistence} Example of persistence diagram computation. The space we consider is a blurry image of a zero, and the function $f$ that we use is the grey level value on each
pixel. We show four different sublevel sets of $f$. For each sublevel set, the corresponding pixels are displayed in pink color. 
%Since the family of sublevel sets is a 
%nested sequence w.r.t.\@ the inclusion (indeed $f^{-1}((-\infty,\alpha])\subseteq f^{-1}((-\infty,\beta])$ for $\alpha\leq\beta$), we keep adding new pixels 
In the first sublevel set, two connected components are present in the sublevel set, so we start two intervals $I_1$ and $I_2$. 
In the second one, one connected component got merged to the other,
so we stop the corresponding interval $I_2$, and a cycle (loop) is created, so we start a third interval $I_3$. 
In the third sublevel set, a new small cycle is created, as well as three more connected
components. In the fourth sublevel set, all pixels belong to the set: all cycles are filled in and all connected components are merged together, so we stop all intervals.
Finally, each interval $I_k$ is represented as a point $P_k$ in the plane (using the endpoints as coordinates). 
}
\end{figure}

%All persistence diagrams in this article are assumed to have finite cardinalities, and points with finite coordinates.
 
\paragraph*{Notation.} Let $\SpD$ be the space of persistence diagrams with countable number of points.
More formally, $\SpD$ can be equivalently defined as a functional space  
$\{m:\R^2\setminus\Delta\rightarrow\N \,:\, {\rm supp}(m) \text{ is countable}\}$, where each point $q\in{\rm supp}(m)$ is a point 
in the corresponding persistence diagram with multiplicity $m(q)$.
Let $\SpND$ be the space of persistence diagrams with less than $N$ points, i.e., $\SpND=\{m:\R^2\setminus\Delta\rightarrow\N \,:\, \sum_q m(q) < N\}$.
%and $\SpNbD$ be the space of bounded persistence diagrams with less than $N$ points.
%
Let $\SpLD$ be the space of persistence diagrams included in $[-L,L]^2$, i.e., $\SpLD=\{m:\R^2\setminus\Delta\rightarrow\N \,:\,{\rm supp}(m) \subset [-L,L]^2\}$.
%and $\SpLfD$ be the space of persistence diagrams with finitely many points included in $[-L,L]$,
%
Finally, let $\SpNLD$ be the space of persistence diagrams with less than $N$ points included in $[-L,L]^2$, i.e., $\SpNLD=\SpND\cap\SpLD$.
Obviously, we have the following sequences of (strict) inclusions: 
$\SpNLD\subset\SpND\subset\SpD$, and
$\SpNLD\subset\SpLD\subset\SpD$.
  
\paragraph*{Diagram distances.} Persistence diagrams can be efficiently compared using the {\em diagram distances}, which is a family of
distances parametrized by an integer $p$ that rely on the computation of {\em partial matchings}. Recall that two persistence diagrams 
$\Dg_1$ and $\Dg_2$ may have different number of points. A {\em partial matching} $\Gamma$ between $\Dg_1$ and $\Dg_2$ is 
a subset of $\Dg_1\times\Dg_2$. It comes along with $\Gamma_1$ (resp. $\Gamma_2$), which is the set of points of $\Dg_1$ (resp. $\Dg_2$) 
that are not matched to a point of $\Dg_2$ (resp. $\Dg_1$) by $\Gamma$.  
The $p$-cost of $\Gamma$ is given as:
\[
c_p(\Gamma) =\sum_{(p,q)\in\Gamma}\|p-q\|_\infty^p + \sum_{p\in\Gamma_1}\|p-\Delta\|_\infty^p + \sum_{q\in\Gamma_2}\|q-\Delta\|_\infty^p.
\]
The $p$-diagram distance is then defined as the cost of the best partial matching:

\begin{defin}
Given two persistence diagrams $\Dg_1$ and $\Dg_2$, the $p$-{\em diagram distance} $d_p$ is defined as:
$$d_p(\Dg_1,\Dg_2)={\rm inf}_{\Gamma}\ \sqrt[\leftroot{-3}\uproot{3}p]{c_p(\Gamma)}.$$
\end{defin}

Note that in the literature, these distances are often called the {\em Wasserstein distances} between persistence diagrams.
Here, we follow the denomination of~\cite{Carriere17e}. In particular, taking a maximum instead of a sum 
in the definition of the cost,
\[
c_\infty(\Gamma) =\max_{(p,q)\in\Gamma}\|p-q\|_\infty + \max_{p\in\Gamma_1}\|p-\Delta\|_\infty + \max_{q\in\Gamma_2}\|q-\Delta\|_\infty.
\]
allows to add one more distance in the family, the {\em bottleneck distance} $d_\infty(\Dg_1,\Dg_2)={\rm inf}_\Gamma\ c_\infty(\Gamma)$.

\paragraph*{Stability.} A useful property of persistence diagrams is {\em stability}. Indeed, it is well known in the literature that
persistence diagrams computed from close functions are close themselves in the bottleneck distance:

\begin{thm}[\cite{Cohen07,Chazal16a}]
Given two tame functions $f,g:X\rightarrow\R$, one has the following inequality:
\begin{equation}\label{eq:stab}
d_\infty(\Dg(f),\Dg(g))\leq \|f-g\|_\infty.
\end{equation}
\end{thm}
In other words, the map $\Dg$ is 1-Lipschitz.
Note that stability results exist as well for the other diagram distances, 
but these results are weaker than the above Lipschitz condition, and they require more conditions---see~\cite{Oudot15}.

\subsection{Bi-Lipschitz embeddings.}
The main question that we adress in this article is the one of preserving the persistence diagram metric properties when using embeddings into Hilbert spaces.
For instance, one may ask the images of persistence diagrams under a feature map into a Hilbert space to be stable as well.
A natural question is then whether a lower bound also holds, i.e.,
whether the feature map $\Phi$ is a {\em bi-Lipschitz embedding} between $(\SpD,d_p)$ and $\HS$.

\begin{defin}
Let $(X,d_X)$ and $(Y,d_Y)$ be two metric spaces. A {\em bi-Lipschitz embedding} between $(X,d_X)$ and $(Y,d_Y)$
is a map $\Phi:X\rightarrow Y$ such that there exist constants $0<A,B<\infty$ such that:
\[ A\,d_X(x,x')\leq d_Y(\Phi(x),\Phi(x'))\leq B\,d_X(x,x'), \]
for any $x,x'\in X$. The metrics $d_X$ and $d_Y$ are called {\em strongly equivalent}, and
the constants $A$ and $B$ are called 
the {\em lower} and {\em upper} metric distortion bounds respectively.
If $A=B=1$, $\Phi$ is called an {\em isometric} embedding. 
\end{defin}

%Note that this implies that $\Phi$ is 
%bijective on ${\rm im}(\Phi)$ and continuous.
%a homeomorphism on ${\rm im}(\Phi)$.  

Note that this definition is equivalent to the commonly used definition that additionally requires $A=\frac1B$.
%As an example, the Sliced Wasserstein distance~\cite{Carriere17e} for persistence diagrams is an approximation of the first diagram distance computed
%by projecting the persistence diagrams onto lines going through the origin. It is strongly equivalent to the first diagram distance $d_1$ on $\SpNLD$.
%However, the lower metric distortion bound goes to $0$ quadratically as $N$ increases. In the following section, we study the behavior of the 
%metric distortion bounds for kernels in general.

\begin{rem}
Finding an isometric embedding of persistence diagrams into a Hilbert space
is impossible since geodesics are unique in a Hilbert space while this is not the case for persistence diagrams,
as shown in the proof of Proposition~2.4 in~\cite{Turner14}.
\end{rem}

\begin{rem}
\label{rmk:SW}
For feature maps that are {\em bounded}, i.e., those maps $\Phi$ such that there exists a constant $C>0$
for which $\|\Phi(\Dg)\|\leq C$ for all $\Dg$, it is obviously impossible to find a bi-Lipschitz embedding. This involves for instance the
Sliced Wasserstein (SW) feature map~\cite{Carriere17e}, which is defined implicitly from a Gaussian-like function. However, note that
if the SW feature map is restricted to a set of persistence diagrams which are close to each other with respect to the SW distance,
then the distance in the Hilbert space corresponding to the SW feature map is actually equivalent to the square root of the SW distance. Hence,
we added the square root of the SW distance in our experiment in Section~\ref{sec:expe}. 
%Indeed, the metric properties proved in this work deal with the so-called Sliced Wasserstein distance,
%which is necessary to define the feature map, but not with the metric properties of the feature map itself.   
\end{rem}

\section{Mapping into separable Hilbert spaces}

%Interestingly, results in the same vein can be obtained when kernels are defined through feature maps onto separable RKHS---as for 
%the Persistence Scale Space~\cite{Reininghaus15} and the Persistence Weighted Gaussian~\cite{Kusano16} kernels---or for Euclidean RKHS---as for the Persistence Images~\cite{Adams17}
%and the Topological Vectors~\cite{Carriere15a}. %We state the results with {\em bi-Lipschitz embeddings}.

In our first main result, we use {\em separability} to determine whether a bi-Lipschitz embedding can exist between the space of persistence diagrams and a Hilbert space.

\begin{defin}
A metric space is called {\em separable} if it has a dense countable subset. 
\end{defin}

For instance, the following three Hilbert spaces (equipped with their canonical metrics) 
are separable: $\R^n$, $\ell_2$ and $L_2(\Omega)$, where $\Omega$ is separable. The two following results 
describe well-known properties of separable spaces. 

\begin{prop}\label{prop:separsub}
Any subspace of a separable metric space is separable as well.
\end{prop}

\begin{prop}\label{prop:separ}
Let $(X,d_X)$ and $(Y,d_Y)$ be two metric spaces, and assume there is a bi-Lipschitz embedding $\Phi:X\rightarrow Y$,
with Lipschitz constants $A$ and $B$.
Then $X$ is separable if and only if ${\rm im}(\Phi)$ is separable.
\end{prop}

%\begin{proof}
%Assume $X$ is separable. Then, there exists a countable subset $X'=\{x_i\}_{i\in\N}$ that is dense in $X$.
%Let us define $Y'=\Phi(X')\subset {\rm im}(\Phi)$. Then, pick any $y=\Phi(x)\in {\rm im}(\Phi)$ and $\epsilon >0$.
%Let $x'\in X'$ such that $d_X(x,x')\leq \frac 1B \epsilon$ and $y'=\Phi(x')$. Then, $d_Y(y,y')\leq B d_X(x,x')\leq \epsilon$.
%Hence $Y'$ is a dense and countable subset of ${\rm im}(\Phi)$.
%The proof is symmetric in $X$ and ${\rm im}(\Phi)$ (with $B$ replaced by $\frac 1A$), hence
%the equivalence.
%\end{proof}

The following lemma shows that for a feature map $\Phi$ which is bi-Lipschitz when restricted to $\SpNLD$,
the limits of the corresponding constants can actually be used to study the general metric distortion in $\SpD$.  

\begin{lem}\label{lem:PDs}
Let $p\in\mathbb{N}^*$ and let $d$ be a metric on persistence diagrams such that $d$ is continuous with respect to $d_p$ on $\SpD$.
%$d$ and $d_p$ are strongly equivalent on $\SpNLD$, 
%with constants $0<A_N^L,B_N^L<+\infty$.
Let 
\begin{align*}
R_N^L&=\left\{\frac{d_p(\Dg,\Dg')}{d(\Dg,\Dg')} : \Dg\neq\Dg'\in\SpNLD\right\}, \\
%let 
A_N^L &= \inf\ R_N^L 
%\left\{d(\Dg,\Dg')/d_p(\Dg,\Dg')\,:\,\Dg\neq\Dg'\in\SpNLD\right\}$ 
\quad \text{and} \quad B_N^L = \sup\ R_N^L.
\end{align*}
%\left\{d(\Dg,\Dg')/d_p(\Dg,\Dg')\,:\,\Dg\neq\Dg'\in\SpNLD\right\}$. 
%
Since $A_N^L$ is nonincreasing  and $B_N^L$ is nonincreasing with respect to $N$ and $L$, we define:
\begin{align*}
A_N&=
%{\rm lim}_{L\rightarrow+\infty}\ A_N^L$, 
\liminf_{L \to \infty} A_N^L, 
\ \ A^L=
%{\rm lim}_{N\rightarrow+\infty}\ A_N^L$, 
\liminf_{N \to \infty} A_N^L, 
\ \ A=
%{\rm lim}_{N,L\rightarrow+\infty}\ A_N^L$, 
\liminf_{N,L \to \infty} A_N^L.\\
B_N&=
%{\rm lim}_{L\rightarrow+\infty}\ A_N^L$, 
\limsup_{L \to \infty} B_N^L, 
\ \ B^L=
%{\rm lim}_{N\rightarrow+\infty}\ A_N^L$, 
\limsup_{N \to \infty} B_N^L, 
\ \ B=
%{\rm lim}_{N,L\rightarrow+\infty}\ A_N^L$, 
\limsup_{N,L \to \infty} B_N^L.
\end{align*}
We define $B_N$, $B^L$, $B$ similarly, since $B_N^L$ is nondecreasing with respect to $N$ and $L$.
Then the following inequalities hold:
\begin{align*}
 A^L\, d(\Dg,\Dg') \leq\, & d_p(\Dg,\Dg') \leq B^L\, d(\Dg,\Dg') &\text{ for all }\Dg,\Dg'&\in\SpLD, \\
 A_N\, d(\Dg,\Dg') \leq\, & d_p(\Dg,\Dg') \leq B_N\, d(\Dg,\Dg') &\text{ for all }\Dg,\Dg'&\in\SpND, \\ %and $d_p$ are strongly equivalent on $\SpD$, with constants
 A\,   d(\Dg,\Dg') \leq\, & d_p(\Dg,\Dg') \leq B\,   d(\Dg,\Dg') &\text{ for all }\Dg,\Dg'&\in\SpD.
\end{align*}

%where $A=\limitinf{N,L} A_N^L$ and $B=\limitinf{N,L} B_N^L$.
\end{lem}

Note that $A$, $A_N$, $A^L$, $B$, $B_N$ and $B^L$ may be equal to $0$ or $+\infty$, so it does not necessarily hold that $d$ and $d_p$ are strongly equivalent on $\SpND$, $\SpLD$ or $\SpD$.

\begin{proof}
%We have shown that $d$ and $d_p$ are not strongly equivalent on $X$ and thus on $\SpD$.
%Assume that $d$ and $d_p$ {\em are} strongly equivalent on $\SpNLD$, with constants $0<A_N^L,B_N^L<+\infty$. 
%Again, let  $\Phi_\sigma:\SpD$ be the feature map of $k_\sigma$ defined on $\SpD$. 
We only prove the last inequality, since the proof extends verbatim to the other two.
Pick any two persistence diagrams $\Dg,\Dg'\in\SpD$. %such that $d_p(\Dg,\Dg')<+\infty$.
Let $\Gamma=\{(p_i,q_i)\}_{i\in\N}$ be an optimal partial matching achieving $d_p(\Dg,\Dg')$, 
where $p_i$ (resp. $q_i$) is either in $\Dg$ (resp. $\Dg'$) or in $\pi_\Delta(\Dg')$ (resp. $\pi_\Delta(\Dg)$).
%Now, let $\tilde D=D\cup\pi_\Delta(D')=\{p_i\}_{i\in\N}$ and $\tilde D'=D'\cup\pi_\Delta(D)=\{q_i\}_{i\in\N}$. 
%Assume w.l.o.g. that the ordering of the points of $\tilde D$ and $\tilde D'$ is consistent with the optimal matching
%$\gamma:\tilde D\rightarrow\tilde D'$ achieving $\distb(D,D')$, 
%i.e. $q_i=\gamma(p_i)$ and $d={\rm sup}\{\|p_i-q_i\|_\infty\,:\,i\in\N\}$.
Given $n\in\N$, we define two sequences of persistence diagrams 
$\{\Dg_n\}_{n\in\N}$ and $\{\Dg'_n\}_{n\in\N}$ recursively with $\Dg_0=\Dg'_0=\emptyset$ and: 
\begin{align*}
\Dg_{n+1}&=\begin{cases} \Dg_n & \text{if }p_{n+1}\in\pi_\Delta(\Dg'), \\ 
\Dg_n\cup\{p_{n+1}\} &\text{otherwise,}\end{cases}
\\
\Dg'_{n+1}&=\begin{cases} \Dg'_n & \text{if }q_{n+1}\in\pi_\Delta(\Dg), \\ 
\Dg'_n\cup\{q_{n+1}\} &\text{otherwise.}\end{cases}
\end{align*}
%
%$\tilde D_n=\{p_i\,:\,i\leq n\}$ and $\tilde D'_n=\{q_i\,:\,i\leq n\}$.
%
Let us define 
\begin{align*}
l_n&=\max\{\max\{\|p\|_\infty:p\in \Dg_n\},\max\{\|q\|_\infty:q\in \Dg'_n\}\}, \\
s_n&=\max\{\card(\Dg_n),\card(\Dg'_n)\},
\end{align*}
%
%Assume that $\limitinf{n} s_n = +\infty$ or $\limitinf{n} l_n = +\infty$. %when $n\rightarrow+\infty$.
Note that both $\{l_n\}_{n\in\N}$ and $\{s_n\}_{n\in\N}$ are nondecreasing. We have $\Dg_n,\Dg'_n\in\mathcal D_{s_n}^{l_n}$ and thus: 
\begin{equation}\label{eq:boundnl}
%A_{s_n}^{l_n}\|\Phi_\sigma(\Dg_n)-\Phi_\sigma(\Dg'_n)\|\leq d_p(\Dg_n,\Dg'_n)\leq B_{s_n}^{l_n}\|\Phi_\sigma(\Dg_n)-\Phi_\sigma(\Dg'_n)\|,
A_{s_n}^{l_n}\,d(\Dg_n,\Dg'_n)\leq d_p(\Dg_n,\Dg'_n)\leq B_{s_n}^{l_n}\,d(\Dg_n,\Dg'_n).
\end{equation}
%by assumption on $\Phi_\sigma$.
%Now, since $d_p(\Dg_n,\Dg)\rightarrow 0$ when $n\rightarrow +\infty$, we have 
%$\|\Phi_\sigma(\Dg_n)-\Phi_\sigma(\Dg)\|\rightarrow 0$ by continuity of $\Phi_\sigma$,
%$d(\Dg_n,\Dg)\rightarrow 0$ by continuity of $d$, 
%and similarly 
%$\|\Phi_\sigma(\Dg'_n)-\Phi_\sigma(\Dg')\|\rightarrow 0$.
%$d(\Dg'_n,\Dg')\rightarrow 0$. 
%Hence, we have $d_p(\Dg_n,\Dg'_n)\rightarrow d_p(\Dg,\Dg')$ and 
%$\|\Phi_\sigma(\Dg_n)-\Phi_\sigma(\Dg'_n)\|\rightarrow \|\Phi_\sigma(\Dg)-\Phi_\sigma(\Dg')\|$ with the triangle inequality.
%$d(\Dg_n,\Dg'_n)\rightarrow d(\Dg,\Dg')$ with the triangle inequality.
%Finally, %letting $A=\limitinf{n} A_{s_n}^{l_n}$ and $B=\limitinf{n} B_{s_n}^{l_n}$, 
%
Assuming $d_p(\Dg_n,\Dg'_n)\rightarrow d_p(\Dg,\Dg')$
\footnote{Note that this is always true if $d_p(\Dg,\Dg')<\infty$. 
Even though this is not clear if this assumption also holds in the general case, 
it is satisfied for the spaces of persistence diagrams defined in our subsequent results Lemma~\ref{lem:notsepar} and Proposition~\ref{prop:fmap}. }, 
it follows that $d(\Dg_n,\Dg'_n)\rightarrow d(\Dg,\Dg')$ by continuity of $d$.
We finally obtain the desired inequality
%\begin{equation}\label{eq:bound}
%A\|\Phi(D)-\Phi(D')\|\leq d_p(\Dg,\Dg')\leq B\|\Phi(D)-\Phi(D')\|,
%A\,d(\Dg,\Dg')\leq d_p(\Dg,\Dg')\leq B\,d(\Dg,\Dg'),
%\end{equation}
by letting $n\rightarrow+\infty$ in~(\ref{eq:boundnl}).
%$\distb(D'_n,D')\rightarrow 0$ and $\distb(D_n,D'_n)\rightarrow d$
%Moreover, $\distb(\tilde D_n,\tilde D'_n\rightarrow d$ when $n\rightarrow+\infty$, and .
%Hence, $\Phi$ is bi-Lipschitz on $(\DS,\distb)$. 
%Now, assume $0<A,B<+\infty$. Then, $\Phi$ is bi-Lipschitz on $(\DS,W_p)$. 
%Since $(\DS,W_p)$ is not separable (Lemma~\ref{lem:notsepar}),
%${\rm im}(\Phi)$ is not separable as well (Proposition~\ref{prop:separ}), which is a contradiction.
%Indeed, we recall that ${\rm im}(\Phi)$ is separable 
%since $\mathcal H$ is (Proposition~\ref{prop:separsub}).
%both $\ell_2$ and $L_2$ are separable (Proposition~\ref{prop:separsub}).
%However,  it follows that ${\rm im}(\Phi)$ is separable from Proposition
%that there cannot be any bi-Lipschitz embedding
%between $\DS$ and either $\ell_2$ or $L_2$. 
\end{proof}

A corollary of the previous results is that even if a feature map taking values in a separable Hilbert space might be bi-Lipschitz when restricted to $\SpNLD$,
the corresponding bounds have to go to 0 or $+\infty$ as soon as the domain of the feature map is not separable. 
%since there cannot be a bi-Lipschitz feature map between a non-separable set of persistence diagrams
%and a separable Hilbert space. 

\begin{thm}\label{thm:th1}
Let $\Phi:\mathcal D_{\Phi}\rightarrow \mathcal H$ be a feature map defined on a non-separable subspace $\mathcal D_{\Phi}$ of persistence diagrams
containing every $\SpNLD$, i.e., $\SpNLD\subset\mathcal D_{\Phi}$ for each $N,L$. Assume $\Phi$ takes values in a 
separable Hilbert space $\mathcal H$, and that $\Phi$ is bi-Lipschitz on each $\SpNLD$ with constants $A_N^L, B_N^L$.
Then either $A_N^L\rightarrow 0$ or $B_N^L\rightarrow +\infty$ when $N,L\rightarrow+\infty$.
\end{thm}

Many feature maps defined in the literature, such as the Persistence Weighted Gaussian feature map~\cite{Kusano17} or the Landscape feature map~\cite{Bubenik15}, 
actually take value in the separable function space $L^2(\Omega)$, where $\Omega$ is the upper half-plane $\{(x,y)\,:\,x\leq y\}$.
Hence, to illustrate how Theorem~\ref{thm:th1} applies to these feature maps, we now provide two lemmata. In the first one,
we define a set $\mathcal S$ which is not separable with respect to $d_1$, and in the second one, we show that $\mathcal S$ is actually included in the 
domain $\mathcal D_{\Phi}$ of these feature maps. 

\begin{lem}\label{lem:notsepar}
Consider the sequence of points $\{p_k=(k, k + \frac 1k)\,:\,k\in\N\}$, and 
define the set $\mathcal S = \{\Dg_u\}_{u\in\mathcal U}\subset \SpD$, where $\mathcal U$ is the 
set of sequences with values in $\{0,1\}$, with: $\Dg_u=\{p_i\,:\,i\in{\rm supp}(u)\}$.
Then $(\mathcal S, d_1)$ is not separable.
%Let $p\in\N^*\cup\{+\infty\}$. Then $(\SpD,d_p)$ is not separable.
\end{lem}

\begin{proof}
First note that since the sequences $u\in \mathcal U$ can have infinite support, the spaces $\mathcal U$ and $\mathcal S = \{\Dg_u\}_{u\in\mathcal U}$ are not countable. 

Let $\sim$ be the equivalence relation on $\mathcal S$ defined with:
\[\Dg_u \sim \Dg_v \Longleftrightarrow {\rm supp}(u)\ \triangle\ {\rm supp}(v) < +\infty,\]
where $\triangle$ denotes the symmetric difference of sets.
Since the set of sequences with finite support is countable, it follows that each equivalence class $[\Dg_u]_{\sim}$ is countable as well.
In particular, this means that the set of equivalence classes $\mathcal S/\sim$ is uncountable, since otherwise $\mathcal S$ would be countable
as a countable union of countable equivalence classes.

We now prove the result by contradiction. Assume that $\mathcal S$ is separable, and let $\mathcal S'\subset\mathcal S$ be the corresponding dense countable subset of $\mathcal S$. 
Let $\epsilon > 0$. Then for each $u\in\mathcal U$, there is at least one sequence $u'\in\mathcal U$ such that $\Dg_{u'}\in\mathcal S'$ and
$d_1(\Dg_u,\Dg_{u'})\leq\epsilon$. We now claim that every such $u'$ satisfies $\Dg_{u'}\in[\Dg_u]_{\sim}$.
Indeed, assume $\Dg_{u'}\not\in[\Dg_u]_{\sim}$ and let ${\mathcal I} = {\rm supp}(u')\ \triangle\ {\rm supp}(u)$. 
Then, since $|\mathcal I|=+\infty$, we would have \[d_1(\Dg_u,\Dg_{u'})=\sum_{i\in\mathcal I}\frac 1i=+\infty > \epsilon,\] which is not possible.
Hence, this means that $|\mathcal S'| \geq |\mathcal S/\sim|$. However, we showed that $\mathcal S/\sim$ is uncountable, meaning that $\mathcal S'$ is uncountable as well,
which leads to a contradiction since $\mathcal S'$ is countable by assumption. 
\end{proof}

We now show that the Persistence Weighted Gaussian and the Landscape feature maps are well-defined on the set $\mathcal S$.
Let us first formally define these feature maps.

\begin{defin}\label{def:landscape}
Given $p =(u,v)\in\R^2$, $u\leq v$, let $\phi_p$ be the triangular function defined with $\phi_p(t)=\frac{v-u}{2}(1-\frac{2}{v-u}|t-\frac{u+v}{2}|)$ if $x\leq t\leq y$ and 0 otherwise.
Then, given a persistence diagram $\Dg$, let $\lambda_k:t\mapsto {\rm kmax}\{\phi_p(t)\}_{p\in\Dg}$, where kmax denotes the $k$-th largest element.
The {\em Landscape} feature map is defined as:
\[
\Phi_{\rm L}:\Dg\mapsto \bar{\lambda},
\quad
\text{where}
\quad
\bar{\lambda}(x,y)=
\begin{cases}
\lambda_{\lceil x \rceil}(y) & x \geq 0,\\
0 & \text{otherwise}.
\end{cases}\]
%We restrict the domain of definition to the set of persistence diagrams that are mapped %by $\Phi_{\rm L}$ 
%to a function of bounded $L^2$ norm: 
%\[\mathcal D_{\Phi^\omega_{\rm L}} = \{ \Dg \in \mathcal D : \| \Phi_{\rm L}(\Dg) \|_{L^2(\R^2)} < \infty \} . \]
\end{defin}

\begin{defin}
Let $\omega:\R^2\rightarrow\R$ be a weight function and $\sigma >0$.
The {\em Persistence Weighted Gaussian} feature map is defined as:
\[
\Phi^\omega_{\rm PWG}:\Dg\mapsto \sum_{p\in\Dg} \omega(p){\rm e}^{-\frac{\|\cdot-p\|_2^2}{2\sigma^2}}.
\]
%Again, we restrict the domain of definition to the set of persistence diagrams that are mapped %by $\Phi^\omega_{\rm PWG}$ 
%to a function of bounded $L^2$ norm: 
%\[\mathcal D_{\Phi^\omega_{\rm PWG}} = \{ \Dg \in \mathcal D : \| \Phi_{\rm PWG}(\Dg) \|_{L^2(\R^2)} < \infty \} . \]

\end{defin}

\begin{prop}\label{prop:fmap}
%Let $\Phi_{\rm L}$ be the feature map associated to the Landscape Kernel.
Let $(x,y)\mapsto (y-x)^2$ be the weight function $(x,y)\mapsto(y-x)^2$.
%and let $\Phi_{\rm PWG}$ be the feature map associated to the corresponding Persistence Weighted Gaussian Kernel.
Let $\mathcal S$ be the set of persistence diagrams defined in Lemma~\ref{lem:notsepar}. Then:
\[
\mathcal S \subset \mathcal D_{\Phi^\omega_{\rm PWG}}\text{ and }\mathcal S \subset \mathcal D_{\Phi_{\rm L}}.
\]
\end{prop}

\begin{proof}
Let $u_k\in\mathcal U$ be the sequence defined with $u_n=1$ if $n\leq k$ and $u_n=0$ otherwise. To show the desired result, it suffices 
to show that $\{\Phi^\omega_{\rm PWG}(\Dg_{u_k})\}_{k\in\mathbb N}$ and $\{\Phi_{\rm L}(\Dg_{u_k})\}_{k\in\mathbb N}$ are Cauchy sequences in $L^2(\R^2)$.
Let $q\geq p\geq 1$, and let us study $\|\Phi(\Dg_{u_q})-\Phi(\Dg_{u_p})\|^2_{L^2(\R^2)}$ for each feature map.

\begin{itemize} 
\item Case $\Phi^\omega_{\rm PWG}$. We have the following inequalities: 

\begin{align*}
\|\Phi^\omega_{\rm PWG}&(\Dg_{u_q})-\Phi^\omega_{\rm PWG}(\Dg_{u_p})\|^2_{L^2(\R^2)}\\
&=\int_{\R^2}\left(\sum_{k=p}^{q}\frac{1}{k^2} {\rm e}^{-\frac{\|x-p_k\|^2_2}{2\sigma^2}}\right)^2{\rm d}x
=\sum_{k=p}^q\sum_{l=p}^q \frac{1}{k^2l^2}\int_{\R^2}{\rm e}^{-\frac{\|x-p_k\|^2_2+\|x-p_l\|^2_2}{2\sigma^2}}{\rm d}x\\
&=\pi\sigma^2\sum_{k=p}^q\sum_{l=p}^q \frac{1}{k^2l^2}{\rm e}^{-\frac{\|p_k-p_l\|^2_2}{4\sigma^2}}\text{ (cf Appendix C in~\cite{Reininghaus14} for a proof of this equality)}\\
&\leq \pi\sigma^2\left(\sum_{k=p}^q \frac{1}{k^2}\right)\left(\sum_{l=p}^q \frac{1}{l^2}\right)
\end{align*}

The result simply follows from the fact that $\{\sum_{k=1}^n\frac{1}{k^2}\}_{n\in\mathbb N}$ is convergent and Cauchy. 

\item Case $\Phi_{\rm L}$. Since all triangular functions, as defined in Definition~\ref{def:landscape}, 
have disjoint support, it follows that %$\Phi_{\rm L}(\Dg_{u_k})$ is comprised of a single function 
the only non-zero lambda function is 
$\lambda_1=\sum_{n=1}^k \phi_n$, where $\phi_n$ is a triangular function defined with
$\phi_n(t)=\frac{1}{2n}(1-|2n(t-(n+\frac{1}{2n}))|)$ if $n\leq t\leq n+\frac 1n$ and 0 otherwise. See Figure~\ref{fig:landscape}.

\begin{figure}[h]\centering
\includegraphics[width=13cm]{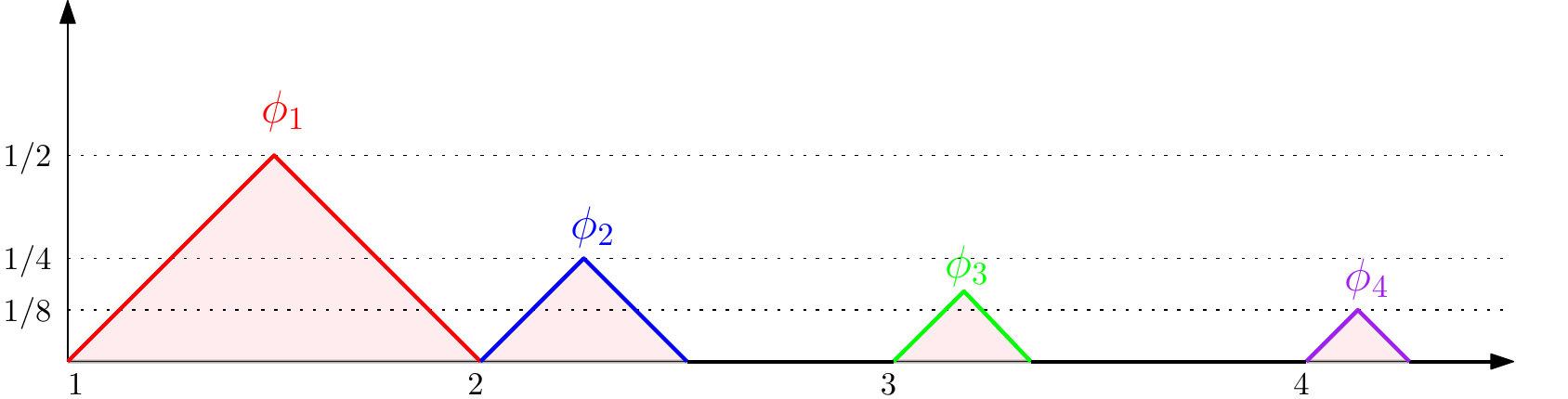}
\caption{\label{fig:landscape} Image of $\Dg_{u_4}$ under $\Phi_{\rm L}$.}
\end{figure}

Hence, we have the following inequalities:

\begin{align*}
\|\Phi_{\rm L}&(\Dg_{u_q})-\Phi_{\rm L}(\Dg_{u_p})\|^2_{L^2(\R^2)}\\
&=\int_{\R}\left(\sum_{k=p}^{q}\phi_k(x)\right)^2{\rm d}x
=\sum_{k=p}^q\sum_{l=p}^q \int_{\R}\phi_k(x)\phi_l(x){\rm d}x\\
&=\sum_{k=p}^q\int_{\R}\phi_k(x)^2{\rm d}x
\leq\sum_{k=p}^q\int_{\R}\phi_k(x){\rm d}x
=\sum_{k=p}^q\frac{1}{4k^2}
\end{align*}
Again, the result follows from the fact that $\{\sum_{k=1}^n\frac{1}{k^2}\}_{n\in\mathbb N}$ is convergent and Cauchy. 
\qedhere
\end{itemize}
\end{proof}

Proposition~\ref{prop:fmap} shows that Theorem~\ref{thm:th1} applies (with the metric $d_1$ between persistence diagrams) 
to the Persistence Weighted Gaussian feature map with weight function $(x,y)\mapsto (y-x)^2$---actually, any
weight function that is equivalent to $(y-x)^2$ when $(x,y)$ goes to 0---and the Landscape feature map. In particular, 
any lower bound for these maps has to go to 0 when $N,L\rightarrow+\infty$ since an upper bound exists for these maps due to their stability properties---see 
Corollary 15 in~\cite{Bubenik15} and Proposition 3.4 in~\cite{Kusano17}.

\section{Mapping into finite-dimensional Hilbert spaces}

In our second main result, we show that more can be said about feature maps into $\R^n$ (equipped with the Euclidean metric), 
using the so-called {\em Assouad dimension}. This involves all vectorization methods for persistence diagrams that we described in the related work.

\paragraph*{Assouad dimension.}
The following definition and example are taken from paragraph 10.13 of~\cite{Heinonen01}. 

\begin{defin}
Let $(X,d_X)$ be a metric space. Given a subset $E\subset X$ and $r>0$, let $N_r(E)$ be the least number of open balls 
of radius less than or equal to $r$ that can cover $E$. 
The {\em Assouad dimension} of $X$ is:
$$
\dimA(X,d_X)={\rm inf}\{\alpha>0\,:\,\exists C >0\text{  s.t. } 
{\rm sup}_{x\in X}N_{\beta r}(B(x,r))\leq C\beta^{-\alpha},\ \forall r>0,\beta\in(0,1]\}.
$$

%\bigcup_{i=1}^{} B\left(x_i,\frac r2 \right)\}$$ 
%called a {\em doubling space}, if for any $x\in X$ and $r>0$, there exists an integer $D\in\N$ and 
%$D$ points $x_1,\dots,x_D\in X$ such that , where $B(x,r)=\{x'\in X\,:\,d_X(x,x')\leq r\}$.
%The infimum over the integers satisfying the previous condition is called the {\em doubling dimension} of $X$.
\end{defin} 
Intuitively, the Assouad dimension measures the number of open balls needed to cover an open ball of larger radius.
For example, the Assouad dimension of $\R^n$ is $n$. Moreover, the Assouad dimension is preserved by bi-Lipschitz embeddings.
%It is a {\em fractal dimension}, and may be equal to $\infty$. 
%while the Assouad dimension of $\ell_2$ and $L_2$ is infinite.

%Let us now look at the evolution of the Assouad dimension through bi-Lipschitz embeddings.

%\begin{prop}\label{prop:Assouadsub}
%Let $(X,d)$ be a metric space and $X'\subseteq X$. %and $d_Y$ be the restriction of $d_X$ on $Y$. 
%Then $\dimA((X',d_X))\leq \dimA((X,d_X))$. 
%\end{prop}

\begin{prop}[Lemma 9.6 in~\cite{Robinson10}]\label{prop:Assouadineq}
Let $(X,d_X)$ and $(Y,d_Y)$ be metric spaces with a bi-Lipschitz embedding $\Phi:X\rightarrow Y$. 
Then $\dimA(X,d_X) = \dimA({\rm im}(\Phi),d_Y)$. 
%Assume $Y$ has finite Assouad dimension $D\in\N$. Then the doubling dimension of $X$ is less than $D$. 
\end{prop}

\paragraph*{Non-embeddability.}
We now show that $\SpNLD$ cannot be embedded into $\R^n$ with bi-Lipschitz embeddings.
The proof of this fact is a consequence of the following lemma:

\begin{lem}\label{lem:Assouadpersistence diagrams}
Let $p\in\N\cup\{\infty\}$, $N\in\N$, and $L>0$.
Then
%$\dimA((\DS_N^L,\distb))=
$\dimA(\SpNLD,d_p)=+\infty$.
\end{lem} 

\begin{proof}
Let $B_p$ %and $B_{\rm w}$ 
denote an open ball with $d_p$.
We want to show that, for any $\alpha>0$ and $C>0$, it is possible to find a persistence diagram $\Dg\in\SpNLD$, a radius $r>0$ and 
a factor $\beta\in(0,1]$ such that the number of open balls of radius at most $\beta r$ needed to cover $B_p(\Dg,r)$ 
%or $B_{\rm w}(D,r)$ 
is strictly larger than $C\beta^{-\alpha}$. 
To this end, we pick arbitrary $\alpha>0$ and $C>0$.
The idea of the proof is to define
$\Dg$ as the empty diagram, and to derive a lower bound on the number of balls with radius $\beta r$ needed to cover $B_p(\Dg,r)$
by considering persistence diagrams with one point evenly distributed on the line $\{(x,x+r):x\in[-L,L]\}$ 
such that the distance between two consecutive points is $r$ in the $\ell_\infty$-distance. Indeed, the pairwise distance between any two such persistence diagrams is
sufficiently large so that they must belong to different balls.
Then we can control the number of persistence diagrams, and thus the number of balls,
by taking $r$ sufficiently small.

More formally, let $M=1+\lfloor C\beta^{-\alpha}\rfloor> C\beta^{-\alpha}$. 
We want to show that we have at least $M$ balls in the cover, meaning that $|\{\Dg_i\}|\geq M$. %, so we take $r=2L/N$. 
Let $r=2L/M$ and $\beta=\frac 12$. We define a cover of $B_p(\Dg,r)$ with open balls of radius less than $\beta r$
centered on a family $\{\Dg_i\}$ as follows:
%In other words, there exist persistence diagrams $\Dg_1,\dots,\Dg_K$, where $K=\left\lfloor(C\beta^{-\alpha}\right\rfloor\geq 1$, such that:
\begin{equation}\label{eq:cover}
B_p(\Dg,r)\subseteq \bigcup_{i} B_p(\Dg_i,\beta r).
\end{equation}

%$r=%\min\{
%\frac{4L}{3+4C\beta^{-\alpha}}
%,\frac L2\}
%$. Note that we have $3r\leq 4L$. 

%We now show that covering $B_p(\Dg,r)$ with open balls of radius at most $\beta r$ 
%requires strictly more than $C\beta^{-\alpha}$ of such balls.
%We proceed by contradiction, and  we assume that there exists a cover of $B_p(\Dg,r)$
%with at most $C\beta^{-\alpha}$ open balls of radius at most $\beta r = \frac r3$.

We now define particular persistence diagrams which all lie in different elements of the cover~(\ref{eq:cover}).
For any $0\leq j \leq M-1$,
%For any $0\leq j \leq \left\lfloor\frac{4L-3r}{4r}\right\rfloor$ (recall that $4L \geq 3r$), %by definition), 
we let $\Dg'_j$ denote the persistence diagram containing only the point $(-L+jr,-L+(j+1)r)$. It is clear that each $\Dg'_j$ is in $\SpNLD$.
See Figure~\ref{fig:proofAssouad}.

\begin{figure}
\centering
\includegraphics[width=5cm]{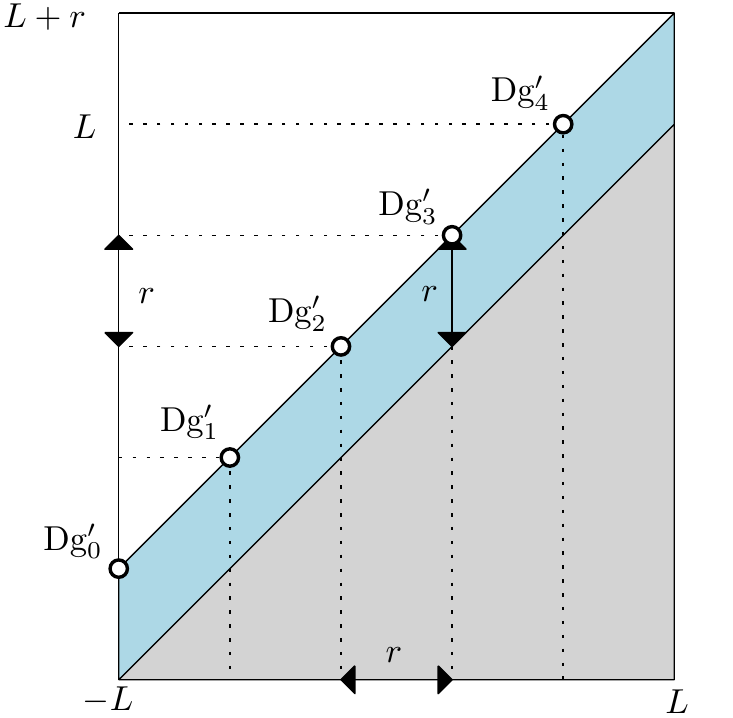}
\caption{\label{fig:proofAssouad} Persistence diagram used in the proof of Lemma~\ref{lem:Assouadpersistence diagrams}. In this particular example, we have $M=5$.}
\end{figure}

%where $r^-$ is a scalar which is strictly smaller but arbitrary close to $r$.
%See Figure~\ref{fig:CE}.
%\begin{figure}[h]\centering
%\includegraphics[width=4cm]{CE}
%\caption{\label{fig:CE}}
%\end{figure}

%Now, we first prove that $\Dg'_j\in\SpNLD$.
%Concerning the abscissae of the point in $\Dg'_j$, we have:
%$$-L\leq -L+2jr\leq -L+\frac{4L-3r}{2r}r\leq L.$$ 

%Concerning its ordinate, according to the previous inequality we have:
%\begin{align*}
%-L & \leq -L+\left(2j+\frac 32\right)r \\
%   & \leq-L+\left(\frac{4L-3r}{2r}+\frac32\right)r= L.
%\end{align*}
%Hence $\Dg'_j\in\SpNLD$. 
Moreover, since
$d_p(\Dg,\Dg'_j)= \frac r2<r$, it also follows that $\Dg'_j\in B_p(\Dg,r)$.
 
%belongs to an element $B_i$ of this cover

Hence, according to~(\ref{eq:cover}), for each $j$ there exists an integer $i_j$  such that $\Dg'_j\in B_p(\Dg_{i_j},\beta r)$.
Finally, note that $j\neq j'\Rightarrow i_j\neq i_{j'}$. Indeed, assuming that there are $j\neq j'$ such that
$i_j=i_{j'}$, and since the distance between $\Dg'_j$ and $\Dg'_{j'}$ is always obtained by matching their points to the diagonal, 
we reach a contradiction with the following application of the triangle inequality:
$$
d_p(\Dg'_j,\Dg'_{j'})=2^{\frac 1p}\frac  r2 
\leq d_p(\Dg'_j,\Dg_{i_j})+d_p(\Dg_{i_j},\Dg_{i_{j'}})+d_p(\Dg_{i_{j'}},\Dg'_{j'}) 
<2\beta r =r.
$$
This observation shows that there are at least $M$
%$1+\left\lfloor\frac{4L-3r}{4r}\right\rfloor$ 
different open balls in the cover~(\ref{eq:cover}), which concludes the proof.
%Hence, we obtain our final contradiction with:
%\[C\beta^{-\alpha}\geq K\geq 1+\left\lfloor\frac{4L-3r}{4r}\right\rfloor >\frac{4L-3r}{4r}=C\beta^{-\alpha}.\qedhere\]
%Note that this proof extends verbatim to the Wasserstein distance $W_p$.
\end{proof}

The following theorem is then a simple consequence of Lemma~\ref{lem:Assouadpersistence diagrams}  and Proposition~\ref{prop:Assouadineq}:

\begin{thm}\label{th:nonembedRn}
Let $p\in\N\cup\{\infty\}$ and $n\in\N$. 
Then, for any $N\in\N$ and $L>0$, there is no bi-Lipschitz embedding between %either $(\DS_N^L,\distb)$ or 
$(\SpNLD,d_p)$ and $\R^n$.
\end{thm}

Interestingly, the integers $N$ and $n$ are independent in Theorem~\ref{th:nonembedRn}: even if one restricts to persistence diagrams
with only one point, it is still impossible to find a bi-Lipschitz embedding into $\R^n$, whatever $n$ is.

\section{Experiments}
\label{sec:expe}
%\begin{figure}\centering
%\includegraphics[width=4cm]{spider}
%\includegraphics[width=4cm]{camel}
%\includegraphics[width=4cm]{outex}
%\includegraphics[width=15cm]{dataset}
%\caption{\label{fig:datasets} Example of persistence diagrams from the three datasets. On the left, we show an example of spider image together with its persistence diagram computed from the x-coordinate function
%defined on the white pixels. In the middle, we show the 3D shape representing a camel and a red point sampled on the shape.
%We also show few levelsets of the geodesic distance function to this point and its corresponding persistence diagram. Finally, we present an example of texture image and its associated persistence %diagram computed
%with the CLBP-S descriptor on the right (red and green colors mean homological dimension 0 and 1 respectively).}
%\end{figure}

In this section, we illustrate our main results by computing the lower metric distortion bounds for the main stable feature maps in the literature. We use
persistence diagrams with increasing number of points to experimentally observe the convergence of this bound to 0, as described in Theorem~\ref{thm:th1}.
More precisely, we generate 100 persistence diagrams for each cardinality in a range going from 10 to 1000 by uniformly sampling points in the unit 
upper half-square $\{(x,y)\,:\,0\leq x,y \leq 1, x\leq y\}$. See Figure~\ref{fig:exPDs} for an illustration.

\begin{figure}[h]\centering
\includegraphics[width=5cm]{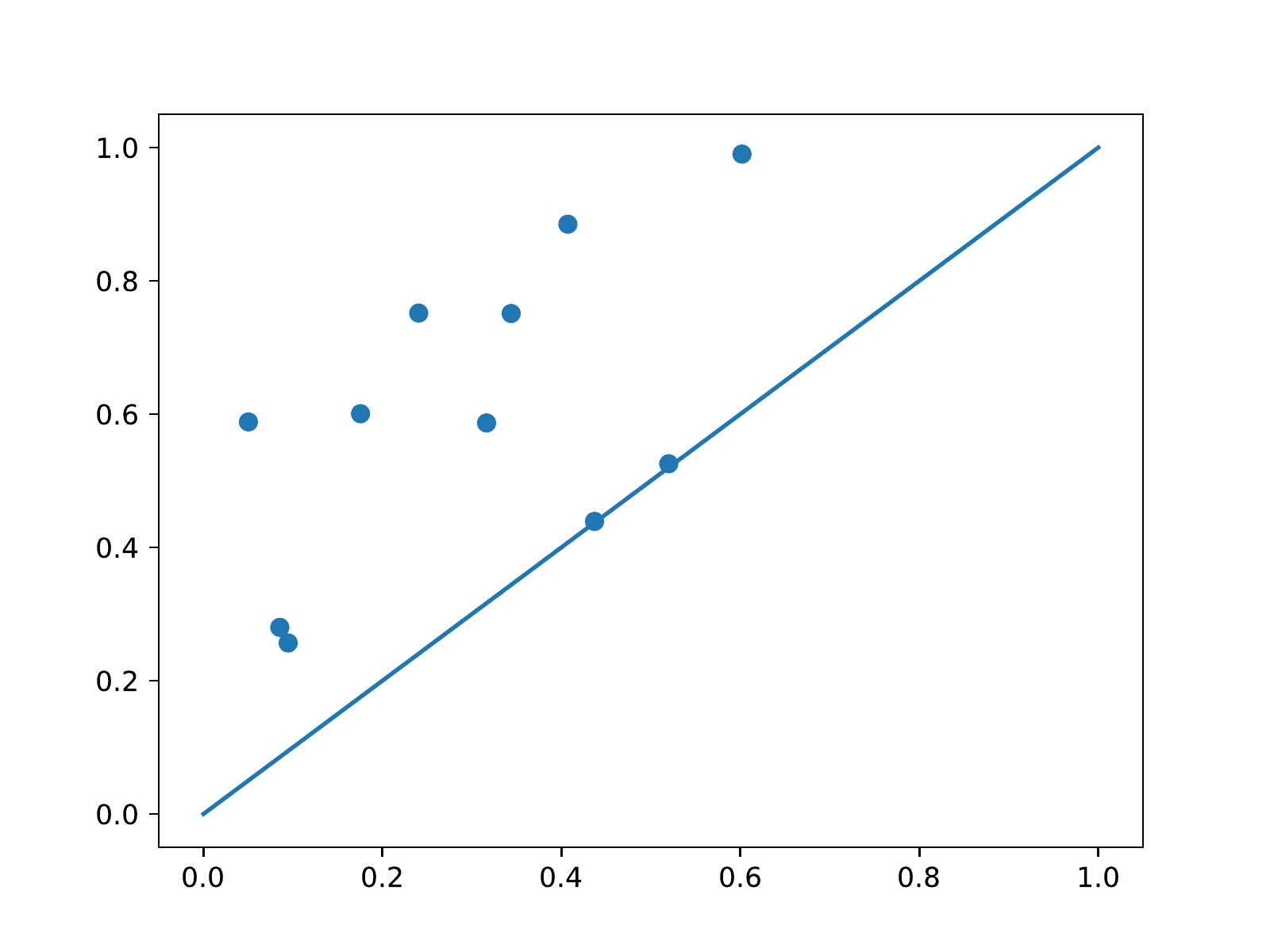}
\includegraphics[width=5cm]{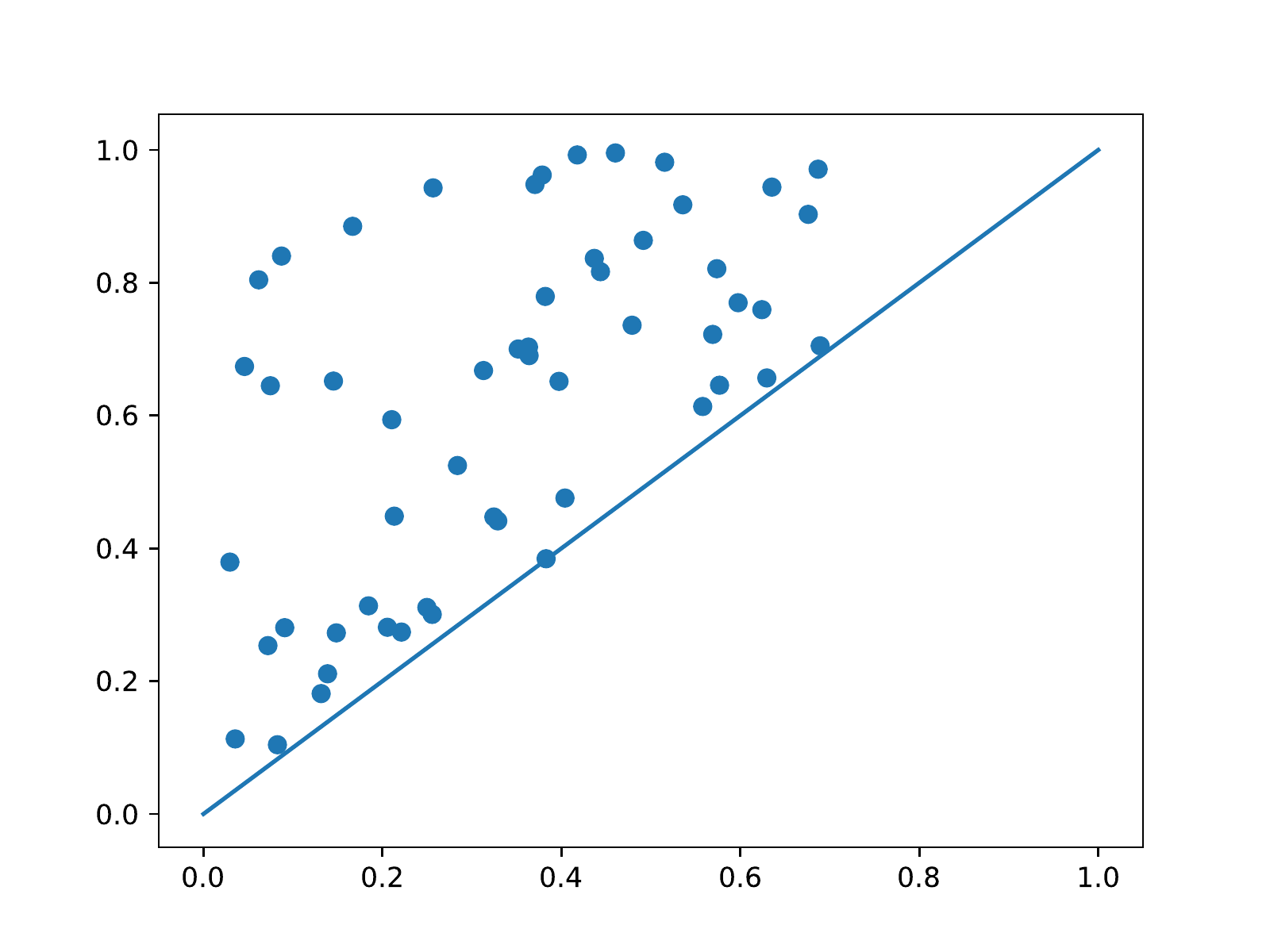}
\includegraphics[width=5cm]{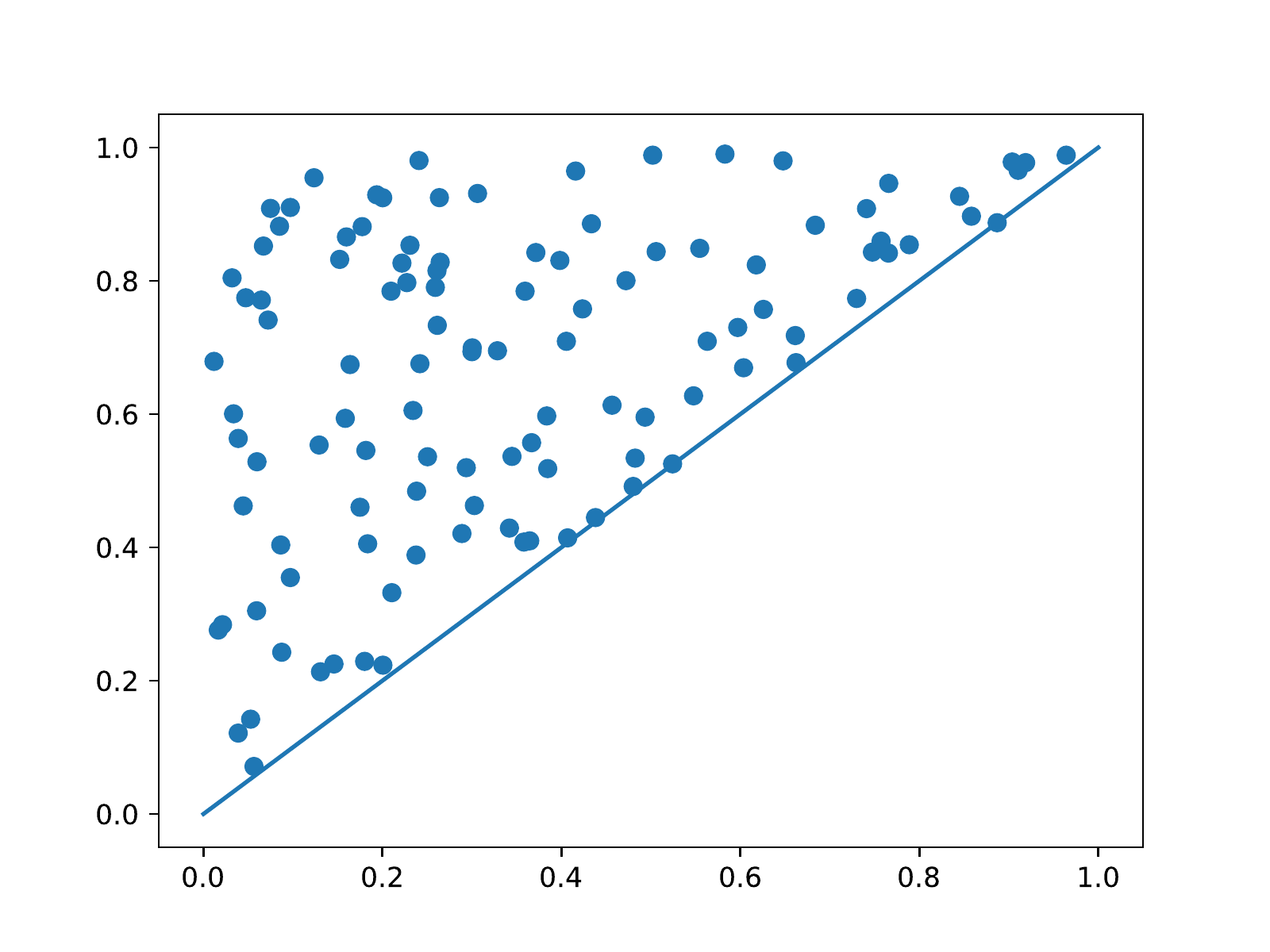}
\caption{\label{fig:exPDs} Example of synthetic persistence diagrams with cardinalities 10 (left), 60 (middle), and 100 (right) generated for the experiment.}
\end{figure}

Then, we consider the following feature maps: 

\begin{itemize}
\item the Persistence Weighted Gaussian with unit bandwidth (PWG)~\cite{Kusano17}, 
\item the Persistence Scale Space with unit bandwidth (PSS)~\cite{Reininghaus15},
\item the Landscape (LS)~\cite{Bubenik15}, 
\item the Persistence Image with resolution $10 \times 10$ and unit bandwidth (IM)~\cite{Adams17} 
\item the Topological Vector with 10 dimensions (TV)~\cite{Carriere15a},
\end{itemize}

Since most of these feature maps enjoy stability properties with respect to the first diagram distance $d_1$, we compute 
the ratios between the metrics in the Hilbert spaces corresponding to these feature maps and $d_1$.
Moreover, we also look at the ratio induced by the square root of the Sliced Wasserstein distance (SW)~\cite{Carriere17e}, as suggested by Remark~\ref{rmk:SW}. 
All feature maps were computed with the sklearn-tda library\footnote{\url{https://github.com/MathieuCarriere/sklearn_tda}}, 
which uses Hera\footnote{\url{https://bitbucket.org/grey_narn/hera}}~\cite{Kerber17} as backend 
to compute the first diagram distances $d_1$ between pairs of persistence diagrams.
These ratios are then displayed as boxplots in Figure~\ref{fig:boxplots}.

\begin{figure}[h!]\centering
\includegraphics[width=6.75cm]{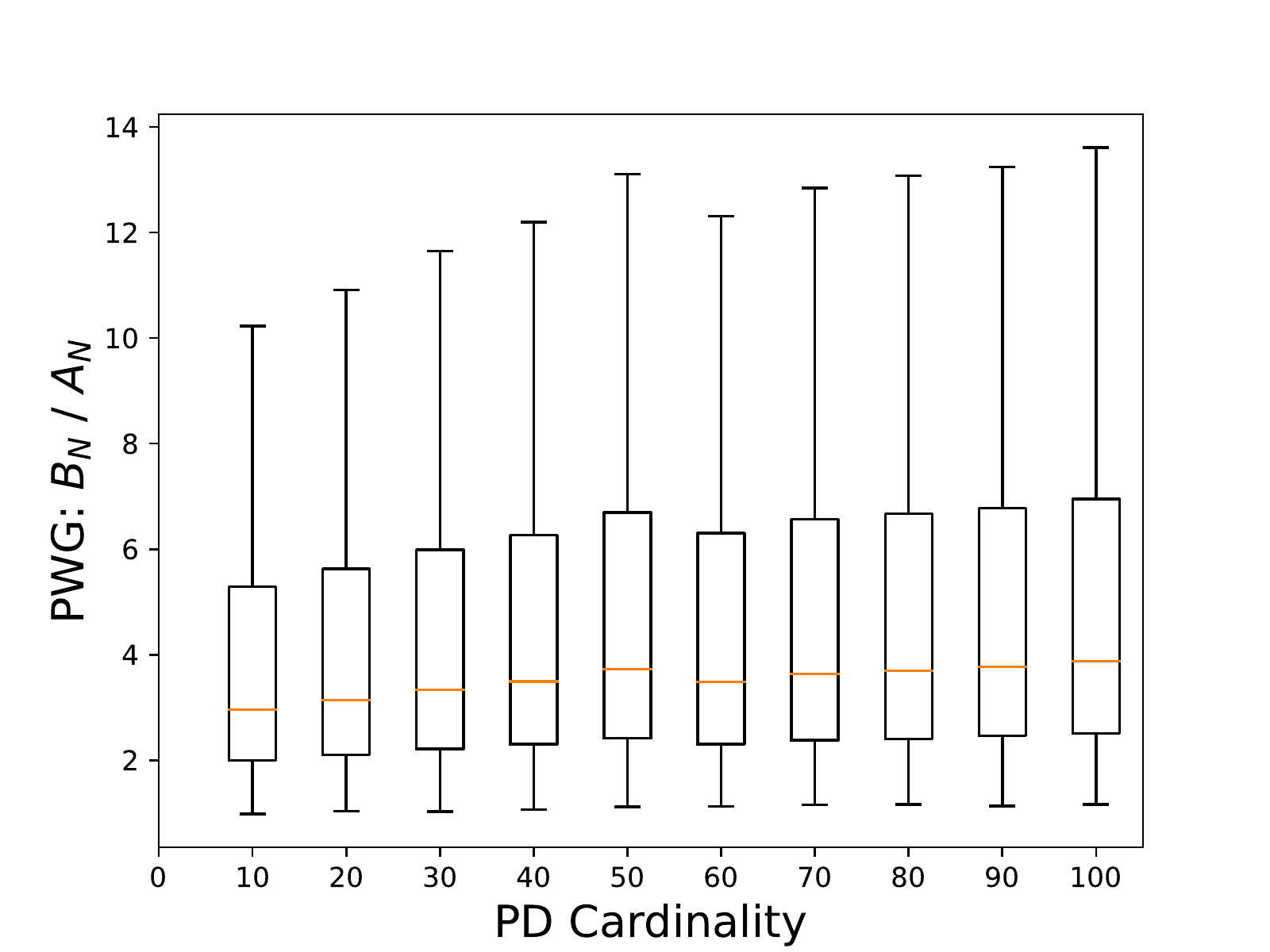}
\includegraphics[width=6.75cm]{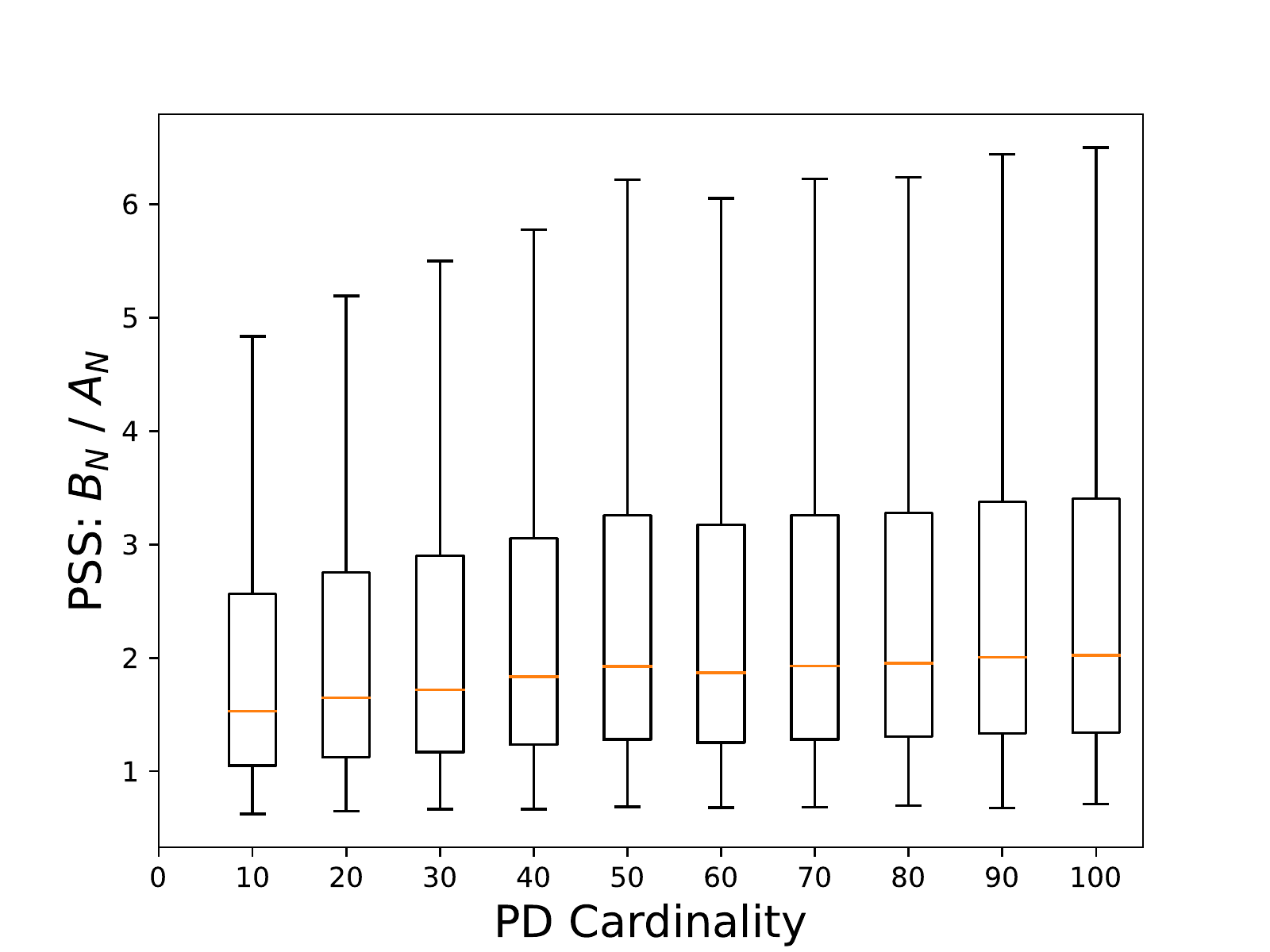}
\includegraphics[width=6.75cm]{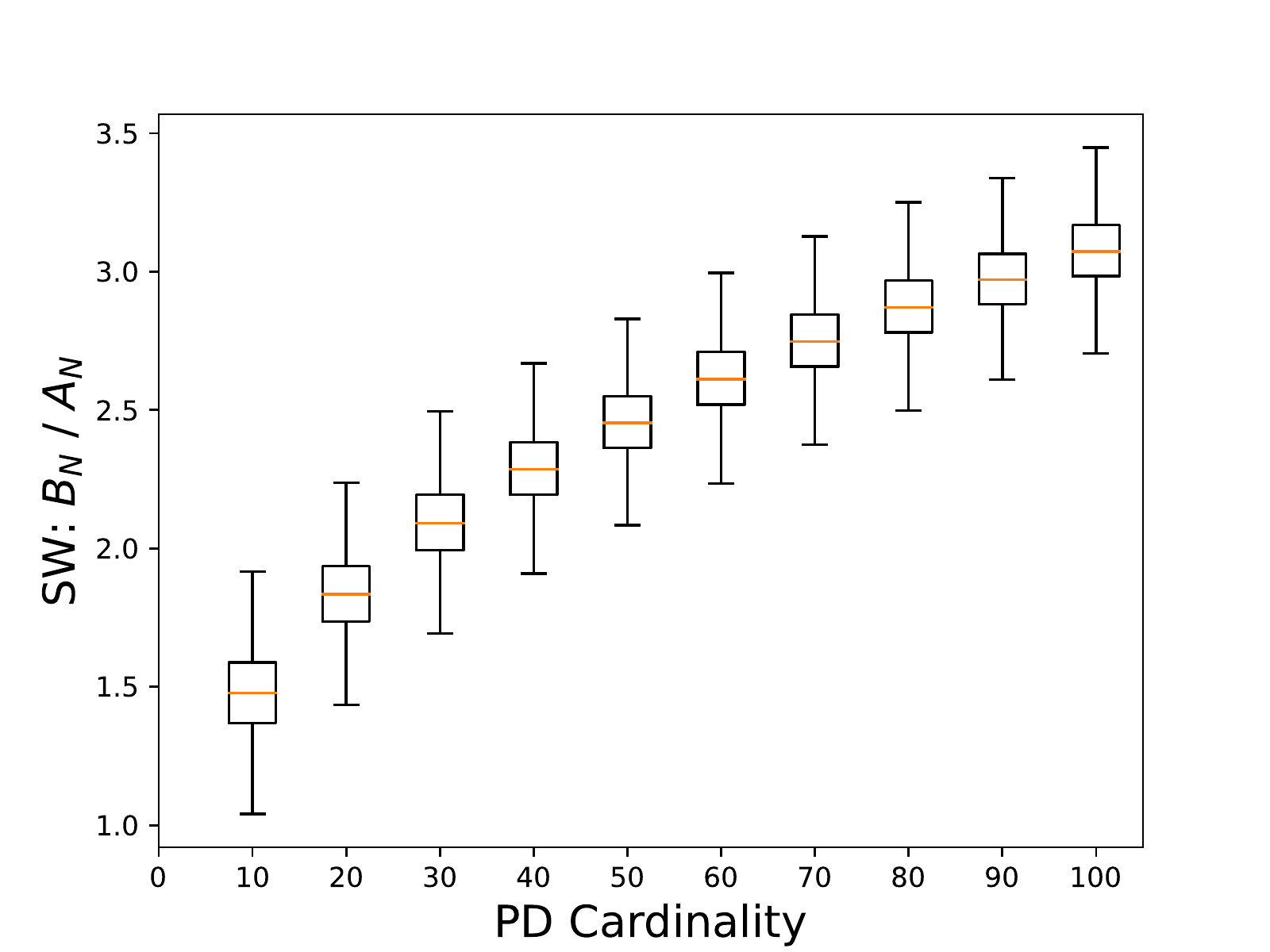}
\includegraphics[width=6.75cm]{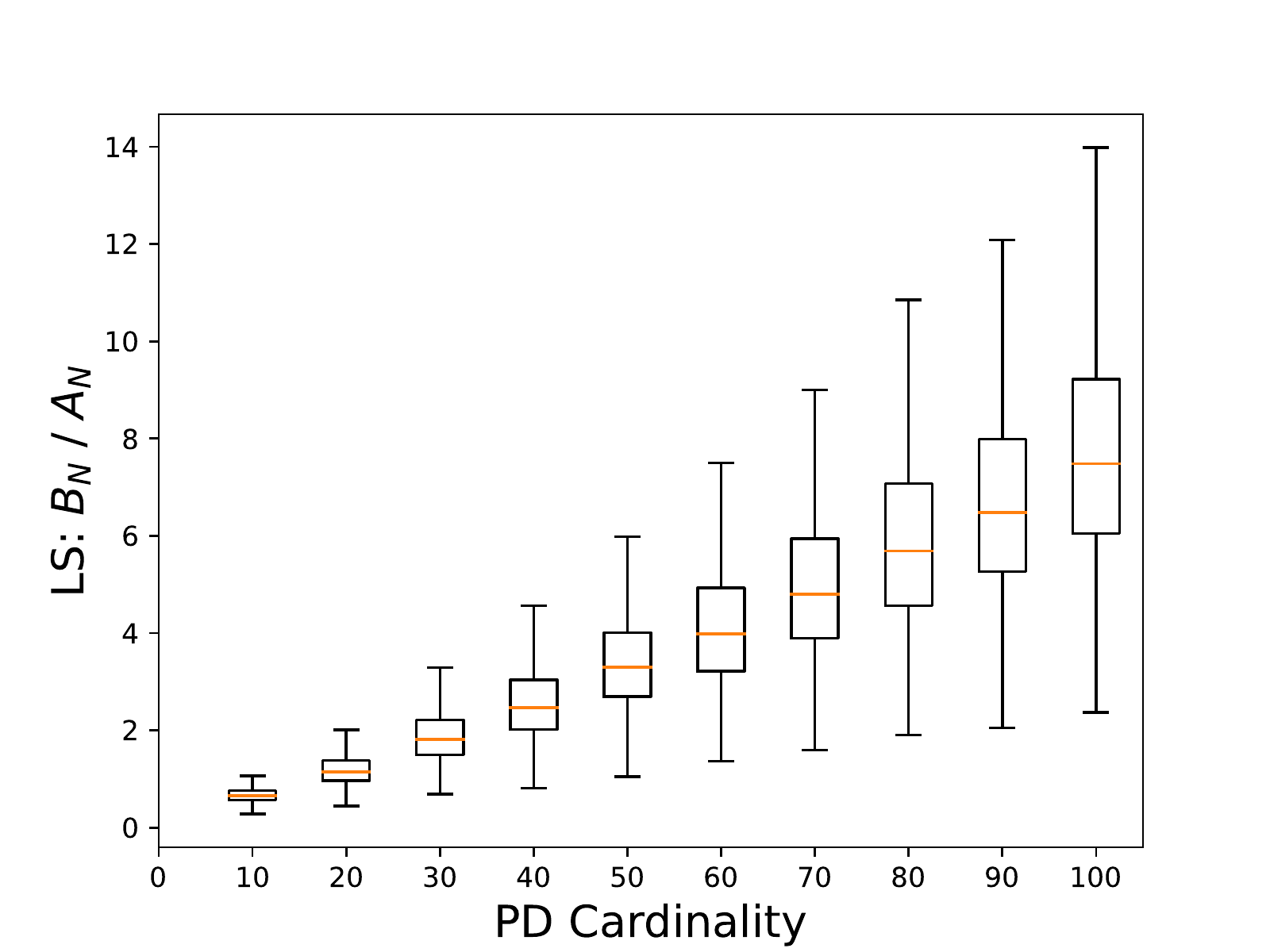}
\includegraphics[width=6.75cm]{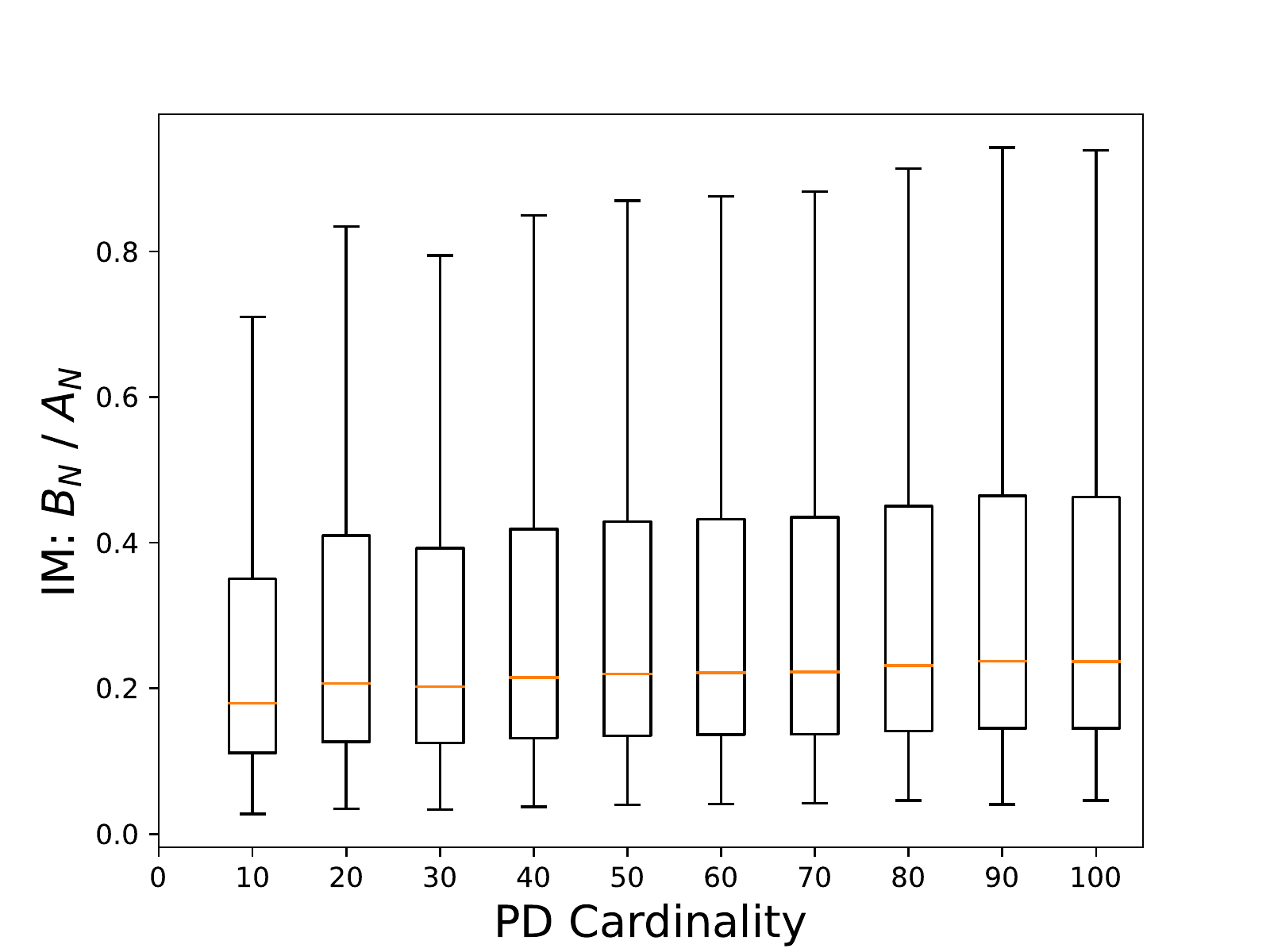}
\includegraphics[width=6.75cm]{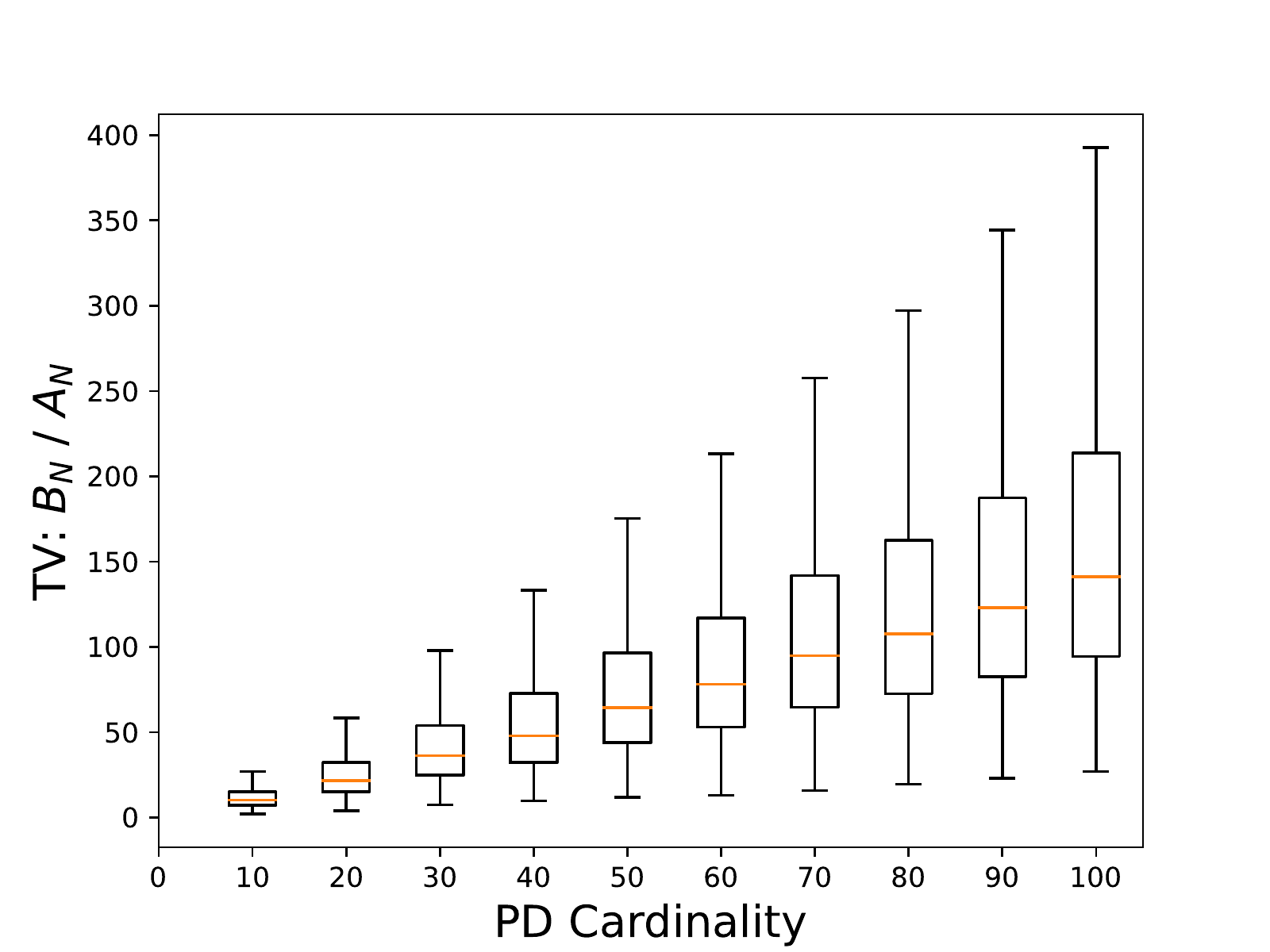}
\caption{\label{fig:boxplots} Boxplots of the ratios between distances induced by various feature maps and the first diagram distance $d_1$.}
\end{figure}

It is clear from Figure~\ref{fig:boxplots} that the extreme values of these ratios (the upper tail of the ratio distributions) increase with the
cardinality of the persistence diagrams, as expected from Theorem~\ref{thm:th1}. This is especially interesting in the case of the Sliced 
Wasserstein distance since the question whether the lower bound that was proved in~\cite{Carriere17e}, which increases with the number of points
in the diagrams, was tight or not, i.e., if a lower bound which is oblivious to the number of points could be derived, is still open. 
Hence, it seems from Figure~\ref{fig:boxplots} that this is not the case empirically.  
It is also interesting to notice that  the divergence speed of these ratios differ from a feature map to another. More precisely, it seems like the metric distortion bounds increase linearly
with the cardinalities for the TV and LS feature maps and the Sliced Wasserstein distance, while it is increasing at a much lower speed for the other feature maps.

\section{Conclusion}

In this article, we provided two important theoretical results about the embedding of persistence diagrams in separable Hilbert spaces, which
is a common technique in TDA to feed machine learning algorithms with persistence diagrams. Indeed, most of the recent attempts
have defined  feature maps for persistence diagrams into Hilbert spaces and showed these maps were stable with respect to the first diagram distance, 
and conjectured whether a lower bound holds as well or not.
In this work, we proved that this is never the case if the Hilbert space is finite dimensional, and that such a lower bound has to go to zero with
the number of points for most other feature maps in the literature. We also provided experiments that confirm this result, by showing a clear
increase of the metric distortion with the number of points for persistence diagrams generated uniformly in the unit upper half-square.

\bibliography{biblio}
\bibliographystyle{alpha}

\end{document}